\documentclass[runningheads]{llncs}

\usepackage{microtype}
\usepackage{graphicx}
\usepackage{subfigure}
\usepackage{booktabs} 
\usepackage[hidelinks]{hyperref}
\usepackage{ulem}
\usepackage{algorithmic}
\usepackage{amsmath,graphicx,amssymb, algorithm, bm,epstopdf,multirow} 
\usepackage{color} 
\usepackage{dsfont} 
\usepackage{tikz-cd}
\usepackage{xfrac} 
\usepackage{array}
\usepackage{hyperref}
\usepackage{cleveref}

\usepackage{enumitem} 

\usepackage{comment} 
\usepackage{appendix}

\newcommand{\mt}[1]{\mathbf{#1}} 
\newcommand{\vt}[1]{\bm{\mathrm{#1}}} 

\definecolor{comment}{gray}{0.4}

\newcommand{\cfd}[1]{{\color{orange} #1}} 
\newcommand{\es}[1]{{\color{red}#1}} 

\newcommand{\y}{\vt{y}} 
\newcommand{\x}{\vt{x}} 
\newcommand{\A}{\mt{A}} 
\newcommand{\rA}{\vt{a}^\T} 
\newcommand{\cA}{\vt{a}} 
\newcommand{\Axi}{[\A\x]_i} 
\newcommand{\yi}{y_i} 
\newcommand{\xj}{x_j} 
\newcommand{\vtheta}{\vt{\theta}} 
\newcommand{\vthetai}{\theta_i} 

\newcommand{\Id}{\mt{I}} 

\newcommand{\R}{\mathbb{R}}
\newcommand{\N}{\mathbb{N}}

\newcommand{\s}[1]{\mathcal{#1}}
\newcommand{\T}{{\sf T}}        

\newcommand{\ind}{\mathds{1}}   

\makeatletter

\newcommand{\dotp}[2]{\langle #1, #2 \rangle}   

\makeatletter

\DeclareMathOperator{\spann}{span}  
\DeclareMathOperator{\pspan}{pos}  

\DeclareMathOperator{\rank}{rank}  

\DeclareMathOperator*{\argmin}{argmin} 
\DeclareMathOperator*{\argmax}{argmax} 
\DeclareMathOperator{\st}{s.t.}

\DeclareMathOperator{\Gap}{Gap}
\DeclareMathOperator{\gap}{gap}
\newcommand{\BV}{\mathrm{BV}} 
\newcommand{\NN}{\mathrm{NN}} 

\newcommand*{\mydots}{.\kern-0.075em.\kern-0.075em.} 
\def\myvdots{\vbox{\baselineskip=3pt \lineskiplimit=0pt 
\kern6pt \hbox{.}\hbox{.}\hbox{.}}} 

\newcommand\blfootnote[1]{%
  \begingroup
  \renewcommand\thefootnote{}\footnote{#1}%
  \addtocounter{footnote}{-1}%
  \endgroup
}

\begin{document}
\title{Accelerating Non-Negative and Bounded-Variable Linear Regression Algorithms with Safe Screening}
\titlerunning{Accelerating Non-Negative and Bounded-Variable Linear Regression}

 \author{Cassio F.~Dantas\inst{1}\orcidID{0000-0002-1934-0625} \and
 Emmanuel Soubies\inst{2}\orcidID{0000-0003-0571-6983} \and
 C\'{e}dric F\'{e}votte\inst{2}\orcidID{0000-0003-3801-5534}}

 \authorrunning{C. F. Dantas et al.}

 \institute{INRAE, TETIS, Univ Montpellier, France. 
 \email{cassio.fraga-dantas@inrae.fr}
 \and
 CNRS, IRIT, University of Toulouse, France
}

\maketitle              

\begin{abstract}
%
Non-negative and bounded-variable 
linear regression problems arise in a variety of applications in machine learning and signal processing.
%
In this paper, we propose a technique to accelerate existing solvers for these problems by identifying saturated coordinates in the course of iterations.
This is akin to safe screening techniques previously proposed for sparsity-regularized regression problems. 
%
The proposed strategy is \emph{provably safe} as it provides theoretical guarantees that the identified coordinates are indeed saturated in the optimal solution.
%
Experimental results on synthetic and real data show compelling accelerations for both non-negative and bounded-variable problems.
\blfootnote{${}^\star$ Code is available at: \url{https://github.com/cassiofragadantas/NN-BV_Screening}} 
\end{abstract}

\section{Introduction}

Due to their fundamental importance in many fields, box-constrained linear regression problems---including {in the large sense} the popular non-negative least squares (NNLS) problem---have received  considerable attention for several decades~\cite{lawson1974solving,stark1995bounded,chen2010nonnegativity}. They are in particular relevant for under-determined linear regression, when the number of variables is larger than the number of measurements. For example, it has been shown that for a variety of problems, the sole non-negativity constraint can be as efficient as sparsity-based regularization~\cite{slawski2013non},\cite[Proposition 4.1]{boyer2019representer}.



Several algorithms have been proposed to address such regression problems. Let us in particular mention the seminal active set approach of~\cite{lawson1974solving} for NNLS, which generalizes to bounded-variable least squares (BVLS) \cite{stark1995bounded}. Accelerated variants were proposed by \cite{bro1997fast,van2004fast}. Other methods for NNLS include projected gradient and Newton methods~\cite{johansson2006application,polyak2015projected}, coordinate descent algorithms~\cite{franc2005sequential}, or interior point approaches~\cite{bellavia2006interior}. They come with different strengths and limitations and we refer the reader to~\cite{chen2010nonnegativity} for a comparative discussion.


\paragraph{Contributions and outline} In this work, we propose a generic strategy to accelerate existing solvers for box-constrained linear regression. It relies on the early identification of saturated coordinates (those attaining the box limits in the solution vector) during the course of iterations, akin to {safe screening} techniques for sparse-regularized regression problems~\cite{Ndiaye2017,Dantas2021}. The general optimization problem addressed in the paper is formulated in~\Cref{sec:pbFormulation}. It takes as a special case the {\em safe squeezing} approach of \cite{Elvira2020} for $\ell_\infty$-regularized problems (see the details in  \Cref{app:L_infinity}). In particular, our framework handles non-symmetric bounds (including an infinite upper bound for the nonnegative case) and can deal with a broad class of data-fidelity functions beyond least squares.
  Then, we show in~\Cref{sec:screening} how some saturated coordinates in the solution vector can be \emph{safely} identified from a primal-dual feasible pair of vectors. This allows us to derive a generic dynamic safe screening procedure for box-constrained linear regression problems. In \Cref{sec:dual_update}, we focus on the dual update step. Although a standard dual-scaling can be deployed for the bounded-variable case, it turns out to be ineffective for the non-negative case. As such, we introduce a new dual update strategy, termed as dual translation, to compute relevant feasible dual points.
Finally, numerical experiments are reported in~\Cref{sec:experiments}. They demonstrate how our proposed approach can significantly accelerate various solvers from the literature.




\paragraph{Notations} 
For $n \!\in\! \mathbb{N}$, we denote by $[n]$ the set $\{1, \dots, n\}$. 
The $i$th entry of a vector $\vt{z} \in \R^n$ is denoted $z_i$ (or sometimes
$[\vt{z}]_i$ to avoid ambiguities).
Given a subset of indices $\s{S} \subseteq [n]$ with cardinality $|\s{S}|= s$,  $\vt{z}_\s{S} \in \R^{s}$ denotes the restriction of $\vt{z}$ to its entries indexed by the elements of $\s{S}$. 
For a matrix $\A \in \R^{m \times n}$, $\cA_j \in \R^m$ stands for its $j$th column
and $\A_{\s{S}} \in \R^{m \times s}$ is the matrix formed out of the columns of $\A$ indexed by the set $\s{S}$. 
Vector inequalities are taken coordinate-wisely, i.e.,   $\vt{a} \leq \vt{b}$ means that $a_i \leq b_i, \forall i$. 
Given two vectors $\vt{l} \in \R^n$ and $\vt{u}\in \R^n$, we define the set ${[\vt{l},\vt{u}]} = [l_1,u_1] \times \cdots \times [l_n,u_n]$.
We use the notation $[\vt{z}]^+$ (resp., $[\vt{z}]^-$)  to refer to the
positive (resp., negative) part operation defined as $\max(0, z_i )$ (resp., $\min(0,z_i)$) for all $i \in [n]$. For a convex function $f : \R^n \rightarrow (-\infty,+\infty]$, $f^*(\vt{u}) = \sup_{\vt{z} \in \R^n} \left< \vt{z}, \vt{u} \right> - f(\vt{z})  $ denotes its Fenchel-Legendre transform (or conjugate function).

\section{Box-constrained linear regression}\label{sec:pbFormulation}

Given a matrix $\A \in \R^{m \times n}$ and a data vector $\y\in \R^m$, we consider the  generic box-constrained  linear regression problem 
\begin{align}\label{eq:LR} 
    \x^\star  \in \argmin_{\x \in \R^n} ~ & P(\x) := \sum_{i=1}^m f(\Axi ;y_i) 
    \\ \st ~ & \vt{l} \leq \x \leq \vt{u} \nonumber
\end{align}
where $\vt{l}\in \R^n$ and $\vt{u} \in \bar{\R}^n$ with $\bar{\R} = \R \cup \{\infty\}$  the Alexandroff extension of the real number line. Given $y$, the function $f(z;y)$ is assumed proper, lower semi-continuous, convex, and differentiable with respect to (w.r.t.) $z$. We also assume it has Lipschitz gradient with constant $1/\alpha$.
We define $F(\vt{z};\y) = \sum_i f(z_i;y_i)$, referred to as {\em loss function}. In the rest of the paper, it will be treated as a function of $\vt{z}$ only. Problem~\eqref{eq:LR} encompasses both  bounded-variable linear regression (BVLR) when $\vt{u} \in \R^n$ and non-negative linear regression (NNLR)\footnote{We use the abbreviation LR (for linear regression) to emphasize that our framework is more general than least-squares (LS) regression and can deal with a broader class of functions $f$ than the quadratic distance used in the LS case.} when $l_i=0$ and $u_i = \infty$ $\forall i\in [n]$. Moreover, our framework can also account for mixed constraints where some entries of $\vt{u}$ are finite and the others are infinite.




\section{Early identification of saturated variables}\label{sec:screening}

Building upon the Gap safe screening technique~\cite{Ndiaye2017} for sparse linear regression, we propose a generic approach (\Cref{alg:solver_screening}) to accelerate solvers for~\eqref{eq:LR} through the early identification of saturated coordinates (i.e.,  identification of $j\in [n]$ such that $x^\star_j=l_j$ or $x^\star_j=u_j$ when $u_j < \infty$).

\subsection{Dual problem}\label{sec:dual_prob}

Let $\s{J}^u_\infty := \{j \in [n] : u_j = \infty\}$ denotes the set of indices for which the upper bound constraint in~\eqref{eq:LR} is infinity.  Then, the dual problem of~\eqref{eq:LR} reads 
\begin{equation} \label{eq:LR_dual}
    \vtheta^\star = \argmax_{\vtheta \in  \s{F}_D}~  D(\vtheta) 
\end{equation}
where the dual objective function and the dual feasible set are  given by (see \Cref{app:dual_problems})
\begin{align} 
    D(\vtheta) &=  -\sum_{i=1}^m f^*(-\vthetai;y_i) - \sum_{j =1 }^n l_j[\A^\T\vtheta]^-_j  \notag \\
&   \quad  \qquad \qquad  \qquad   \qquad   - \sum_{j \in [n]\backslash  \s{J}^u_\infty} u_j [\A^\T\vtheta]^+_j ,  \label{eq:Dual_obj} \\
    \s{F}_D &= \{ \vtheta \in \R^m : \forall j \in \s{J}^u_\infty, \;  \cA_j^\T \vtheta \leq 0 \}. \label{eq:Dual_set}
\end{align}
Note that the dual solution $\vtheta^\star$ is unique (thanks to the differentiability of $f$).

From~\eqref{eq:Dual_obj} and~\eqref{eq:Dual_set}, we see that when $\vt{u} \in \R^n$ (BVLR), the dual problem is unconstrained, that is $\s{F}_D= \R^m$. In contrast, when $\vt{l} = \mathbf{0}$  and all the entries of $\vt{u}$ are equal to $\infty$ (NNLR), the dual cost function simplifies as $D(\vtheta) = - \sum_{i=1}^m f_i^*(-\vthetai;y_i)$ and the dual feasible set reads $\s{F}_D = \{ \vtheta \in \R^m : \A^\T \vtheta \leq 0 \}$.


\begin{remark}
{Note that if $\vt{u} = -\vt{l} = \delta \vt{1}_m$, the second and third terms of \eqref{eq:Dual_obj} sum up to $\delta \|\A^\T \vtheta\|_1$ and we obtain a form of Lasso problem. 
More generally, the dual problem of BVLR  can be seen as a further generalization of the Lasso, where the penalization on the vector $\A^\T \vtheta \in \R^n$ is weighted by $\vt{l}$ (resp., $\vt{u}$) for its negative (resp., positive) entries.}
\end{remark}


\subsection{Safe identification of saturated variables}

From the primal~\eqref{eq:LR} and dual~\eqref{eq:LR_dual} problems, the first-order primal-dual  optimality conditions for a primal-dual solution pair $(\x^\star,\vtheta^\star) \in [\vt{l},\vt{u}] \times \s{F}_D$ are given by (see \Cref{sec:opt_cond}):
\begin{align}
    \label{eq:optimality1}
    &\forall i \in [m],~ \vthetai^\star = -f'([\A\x^\star]_i ; y_i) ,\\
    \label{eq:optimality2}
    &\forall j \in [n],~ \cA_j^\T \vtheta^\star \in 
    \left\lbrace
    \begin{array}{cl}
         \left(-\infty, 0 \right] & \text{if } ~\xj^\star= l_j \\
         \left[ 0, +\infty \right) & \text{if } ~\xj^\star= u_j \\
         \{0\} & \text{if } ~\xj^\star  \in (l_j, u_j) 
    \end{array} \right. .
\end{align}
Let us emphasize that, for NNLR, the second case of~\eqref{eq:optimality2} never occurs (as $\xj^\star < \infty = u_j$).

Upon knowledge of the dual solution $\vtheta^\star$,  the optimality condition \eqref{eq:optimality2} (known as sub-differential inclusion) constitutes a natural criterion to identify saturated coordinates of the primal solution $\x^\star$. More precisely, we have
\begin{equation}\label{eq:screen_cond}
\begin{array}{rl}
            \displaystyle    \forall j \in [n], \; \cA_j^\T\vtheta^\star < 0   &\implies \xj^\star = l_j ,\\ 
        \displaystyle  \forall j \in [n]\backslash \s{J}_\infty^u, \;  \cA_j^\T\vtheta^\star > 0    &\implies \xj^\star = u_j.
\end{array}
\end{equation}
Although this  criterion cannot be used in practice as the dual solution  $\vtheta^\star$ is not known in advance, a relaxed version of it can be obtained with only a  partial knowledge of the location of $\vtheta^\star$. More precisely, upon knowledge of a region $\s{R}\subset \R^m$ such that  $\vtheta^\star \!\in\! \s{R}$---referred to as \textit{safe region}---we can define a weaker version of~\eqref{eq:screen_cond} as 
\begin{equation}\label{eq:screen_cond_R}
    \begin{array}{rl}
     \displaystyle   \forall j \in [n], \; \max_{\vtheta' \in \s{R} } \, \cA_j^\T\vtheta'  < 0   &\implies \xj^\star =  l_j ,\\
    \displaystyle  \forall j \in [n]\backslash \s{J}_\infty^u, \;  \min_{\vtheta' \in \s{R} } \, \cA_j^\T\vtheta' > 0   &\implies \xj^\star =  u_j .
    \end{array} 
\end{equation}
Clearly, the smaller the region $\s{R} \ni \vtheta^\star$, the larger the number of saturated variables that can be identified. A convenient and efficient
choice for $\s{R}$ is presented in the next section.

\subsection{The Gap safe sphere}

The Gap safe sphere---initially proposed by \cite{Ndiaye2017} in the context of sparse linear regression---is defined  for any primal-dual pair  $(\x,\vtheta) \in [\vt{l},\vt{u}] \times \s{F}_D$ by 
\begin{align}
\s{B}(\vtheta,r),  \text{ with } r = \sqrt{\frac{2\Gap(\x,\vtheta)}{\alpha}}
\label{eq:GAP_Safe_sph}
\end{align}
where the duality gap is given by
\begin{align} \label{eq:Dual_Gap}
    \Gap(\x,\vtheta) = P(\x) - D(\vtheta).
\end{align}
Let us recall that $\alpha$ in~\eqref{eq:GAP_Safe_sph} is the inverse of the Lipschitz constant of the gradient of $f$. Equivalently, this means that $-f^\star$ (and hence $D$) is $\alpha$-strongly concave. This safe region
 leads to state-of-the-art screening performances due, in particular, to two key properties.
\begin{itemize}
    \item A simple geometry which allows to simplify~\eqref{eq:screen_cond_R} as
    \begin{equation}\label{eq:safe_rule}
    \begin{array}{rl}
          \displaystyle \forall j \in [n], \; \cA_j^\T\vtheta < - r\|\cA_j\|_2  &\implies   \xj^\star =  l_j, \\
          \displaystyle \forall j \in [n]\backslash \s{J}_\infty^u, \;  \cA_j^\T\vtheta > \phantom{-} r\|\cA_j\|_2   &\implies   \xj^\star =  u_j.
    \end{array} 
     \end{equation}
    This significantly limits the computational overhead that is introduced when testing the validity of this criterion (lines~\ref{alg:line:L_sat_var} and~\ref{alg:line:U_sat_var} in \Cref{alg:solver_screening}).
    \item It vanishes when the duality gap tends to zero. In other words, if strong duality holds (i.e., $P(\x^\star) - D(\vtheta^\star) =0$), then the radius of the Gap safe sphere vanishes as the iterates $\{\x^k,\vtheta^k\}$ converge to $\{\x^\star,\vtheta^\star\}$. 
\end{itemize}

A proof that $\s{B}(\vtheta,r)$ is indeed a safe region (i.e., that $\vtheta^\star \in \s{B}(\vtheta,r)$) was provided by~\cite[Theorem 6]{Ndiaye2017} in the context of sparse linear regression. Actually, this proof can be directly applied to the linear regression problems with box constraints considered in the present paper.

\begin{remark}
Note that the Gap safe sphere can also be defined using \emph{local} strong concavity bounds of $D$ computed on well-chosen subsets of the domain~\cite[Theorem 5]{Dantas2021}. This would allow to extend the applicability of \Cref{alg:solver_screening} (following~\cite{Dantas2021}) to a more general class of functions $f$ such as the $\beta$-divergences with $\beta \in [1,2)$, that includes in particular the popular Kullback-Leibler divergence \cite{cic10}.
\end{remark}

\subsection{Resulting screening algorithm}

The proposed safe screening approach for Problem~\eqref{eq:LR} is presented in 
\Cref{alg:solver_screening}. It can be deployed with any iterative solver for~\eqref{eq:LR}, as indicated by the generic notation 
$$\{\x,\vt{\eta}\} \leftarrow \mathtt{PrimalUpdate}(F(\A \cdot + \vt{z};\y); \x,\vt{\eta}).$$
This has to be understood as performing few iterations of a given primal solver on $F(\A \cdot + \vt{z};\y)$  from the initial point~$\x$.
The vector $\vt{\eta}$ contains the hyperparameters of the solver  (e.g.,   step sizes)  and we will make the role of  $\vt{z}\in \R^m$ explicit hereafter.

In \Cref{alg:solver_screening}, {$\Theta : \R^n \rightarrow \s{F}_D$ denotes a function that computes a dual feasible point from the current primal point~$\x$. Details on how to define this function will be provided in \Cref{sec:dual_update}. The quantity}  $\s{A} \subseteq [n]$ refers the \emph{preserved set}, which is the complement of the set of screened coordinates. Starting from $\s{A} = [n]$, it is dynamically reduced (line~\ref{alg:line:preserved_set}) by removing the components that are surely identified as being saturated in the solution vector $\x^*$ (lines~\ref{alg:line:L_sat_var} and~\ref{alg:line:U_sat_var}), according to the safe rule~\eqref{eq:safe_rule}. At the same time, these identified saturated components in $\x$ are permanently set to their optimal value (line~\ref{alg:line:fix_saturated_var}), and their contribution to the vector of measurements is stored in $\vt{z} \in \R^m$ (line~\ref{alg:line:store_z}). 

It follows that the computation of
\begin{equation}\label{eq:decmposition_fwd}
    \A \x = \A_\s{A} \x_{\s{A}} + \A_{\s{A}^c} \x_{\s{A}^c} = \A_\s{A} \x_{\s{A}} + \vt{z}
\end{equation}
 can be reduced from $O(mn)$ to $O(m(|\s{A}|+1))$ in the next calls to the primal solver. As such, the more saturated components are identified, the larger is the speed improvement in the next calls to the primal solver.

\begin{remark}
The introduction of the variable $\vt{z}$ in \Cref{alg:solver_screening} is convenient to present a generic algorithm that encompasses the complete class of functions $f$ considered in this work. Yet, this additional variable can be discarded for some loss functions. For example, when $F(\A\x; \y) = \|\A\x - \y\|^2_2$,  line~\ref{alg:line:store_z} can be replaced by $\y \gets \y - \A_{\s{S}_l \cup \s{S}_u} \x_{\s{S}_l \cup \s{S}_u}$, thus avoiding the use of $\vt{z}$ as well as the addition operation in~\eqref{eq:decmposition_fwd}.
\end{remark}

It is worth mentioning that the acceleration of the primal iterates provided by the screening procedure has to be balanced with the computational overhead of the screening step itself. This mainly concerns the cost of computing the inner products $\cA_j^T \vtheta$ at lines~\ref{alg:line:L_sat_var} and~\ref{alg:line:U_sat_var} which, all together, have a complexity of $O(|\s{A}|m)$. Fortunately, {provided a suitable choice of the dual update function $\Theta(\x)$ (see~\Cref{sec:dual_update}),  most standard primal solvers already require the computation of these inner products. They can thus be reused for free in the screening step}.


\begin{algorithm}[tb]
\caption{Generic screening procedure for Problem~\eqref{eq:LR}} \label{alg:solver_screening}
\begin{algorithmic}[1]
\STATE \textbf{Initialize} $\mathcal{A}= [n]$, $\x \in [\vt{l}, \vt{u}]$, $ \varepsilon_{\gap}>0$ 
\STATE \textbf{Set}  $\vt{\eta}$ according to the solver, 
\STATE \textbf{Compute} $\alpha >0$, a strong concavity bound of $D$
\STATE \textbf{Initialize}  $\vt{z}= \mathbf{0}$ 
	\REPEAT
		\STATE \textit{--- Solver update restricted to the preserved set  ---}
		\STATE $\{\x_{\s{A}},\vt{\eta}\} \mkern-4mu \gets \mkern-4mu  \mathtt{PrimalUpdate}(F(\A_{\s{A}} \cdot + \vt{z}; \y) ; \x_{\s{A}},\vt{\eta})$
		\STATE \textit{--- Dynamic safe screening ---}
		\STATE $\vtheta \gets \Theta(\x) \in \s{F}_D$ \COMMENT{Dual update}\label{alg:line:dual_up}
		\STATE $r \gets \sqrt{2\Gap(\x,\vtheta)/\alpha}$ \COMMENT{Safe radius}
		\STATE $\s{S}_l  \mkern-3mu  \gets \mkern-3mu \{j \!\in\! \s{A}  ~|~ \cA_j^\T\vtheta < - r \|\cA_j\|_2 \}$ \COMMENT{Lower-saturated set}\label{alg:line:L_sat_var}
		\STATE $\s{S}_u \mkern-5mu  \gets \mkern-4mu \{j \!\in\! \s{A} \backslash \s{J}_\infty^u  | \cA_j^\T\vtheta \mkern-3mu > \mkern-3mu r \|\cA_j\|_2 \}$ \COMMENT{Upper-saturated set}\label{alg:line:U_sat_var}
		%
		\STATE $\x_{\s{S}_l} \gets \vt{l}_{\s{S}_l}$, $\x_{\s{S}_u} \gets \vt{u}_{\s{S}_u}$ \COMMENT{Set of saturated entries}\label{alg:line:fix_saturated_var}
        %
         \STATE $\vt{z} \gets \vt{z} + \A_{\s{S}_l \cup \s{S}_u} \x_{\s{S}_l \cup \s{S}_u}$ \COMMENT{Update saturated part}\label{alg:line:store_z}
		 \STATE $\s{A} \gets \s{A} \backslash (\s{S}_l \cup \s{S}_u)$ \COMMENT{Update preserved set} \label{alg:line:preserved_set}
	\UNTIL{$\Gap(\x,\vtheta) < \varepsilon_{\gap}$} 
\end{algorithmic}
\end{algorithm}

\begin{algorithm}[tb]
\caption{Screening procedure for NNLR} \label{alg:solver_screening_NNLR}
\begin{algorithmic}[1]
\STATE \textbf{Initialize} $\mathcal{A}= [n]$, $\x \geq 0$, $ \varepsilon_{\gap}>0$
\STATE \textbf{Set}  $\vt{\eta}$ according to the solver, 
\STATE \textbf{Compute} $\alpha >0$, a strong concavity bound of $D$
	\REPEAT
		\STATE \textit{--- Solver update restricted to the preserved set  ---}
		\STATE $\{\x_{\s{A}},\vt{\eta}\} \mkern-4mu \gets \mkern-4mu  \mathtt{PrimalUpdate}(F(\A_{\s{A}} \cdot; \y) ; \x_{\s{A}},\vt{\eta})$
		\STATE \textit{--- Dynamic safe screening ---}
		\STATE $\vtheta \gets \Theta(\x) \in \s{F}_D$ \COMMENT{Dual update}
		\STATE $r \gets \sqrt{2\Gap(\x,\vtheta)/\alpha}$ \COMMENT{Safe radius}
		\STATE $\s{S}_l  \mkern-3mu  \gets \mkern-3mu \{j \!\in\! \s{A}  ~|~ \cA_j^\T\vtheta < - r \|\cA_j\|_2 \}$ \COMMENT{Lower-saturated set}
		\STATE $\x_{\s{S}_l} \gets \mathbf{0}$ \COMMENT{Set of saturated entries}
		 \STATE $\s{A} \gets \s{A} \backslash \s{S}_l $ \COMMENT{Update preserved set}
	\UNTIL{$\Gap(\x,\vtheta) < \varepsilon_{\gap}$} 
\end{algorithmic}
\end{algorithm}

Finally, let us emphasize that \Cref{alg:solver_screening} can be simplified in the NNLR case as shown in~\Cref{alg:solver_screening_NNLR}. Indeed, as $\s{J}_\infty^u = [n]$, we always have $\s{S}_u=\emptyset$. Moreover, because $\vt{l} = \mathbf{0}$, the vector $\vt{z}$ remains always zero. 

\section{Computing a dual feasible  point}\label{sec:dual_update}

A crucial step of the proposed screening procedure lies in the computation of a dual feasible point (function ${\Theta} : \R^n \rightarrow \s{F}_D$ at line~\ref{alg:line:dual_up} of \Cref{alg:solver_screening}). One can see from the definition of the Gap safe sphere in~\eqref{eq:GAP_Safe_sph} that the closer $ {\Theta}(\x)$ is to $\vtheta^\star$, the smaller the safe region is likely to be. Moreover, the computation of ${\Theta}(\x)$ should be cheap in order to minimize the computational load of the screening step. In this section, we present ways of defining ${\Theta}$ so as to meet these two desirable properties.
Without loss of generality, we focus only on the BVLR and NNLR cases, the case with mixed constraints being easily deduced from the latter two. 

\subsection{Dual update for BVLR}

As pointed out in \Cref{sec:dual_prob}, the dual problem of BVLR is unconstrained, i.e., $\s{F}_D = \R^m$. As such, any point $\vtheta \in \R^m$ is admissible to be a center for the Gap safe sphere. Following the \textit{dual scaling} idea~\cite{Ndiaye2017}, we propose to define $\Theta$~as
\begin{align} \label{eq:BVLR_dual_up}
\Theta(\x) := -\nabla F(\A\x ; \y).
\end{align}
The notation $\nabla F(\A\x ; \y)$ refers to the gradient of the function $F(\cdot;\y)$ evaluated in $\A\x$. Note that here no scaling of this gradient is required as  $\s{F}_D = \R^m$.
The rationale behind this choice is twofold. First, we get from the primal-dual link~\eqref{eq:optimality1} that  $\Theta(\x) \rightarrow \vtheta^\star$ as $\x \rightarrow \x^\star$. Second, for any  first-order primal solver,  
\begin{equation}
    \label{eq:grad}
    \nabla P(\x) = \A^\T \nabla F(\A\x; \y) = - \A^\T \Theta(\x)
\end{equation}
is computed during the primal update step. This vector contains nothing else than the inner products needed for the screening test (lines~\ref{alg:line:L_sat_var} and~\ref{alg:line:U_sat_var} of~\Cref{alg:solver_screening}) and can thus be reused for free.

\subsection{Dual update for NNLR}

Computing a dual feasible point for the NNLR problem is more involved than for BVLR. Here, $\s{F}_D = \{ \vtheta \in \R^m : \A^\T \vtheta \leq 0\}$ and dual scaling is no longer possible since, 
\begin{equation}
    \vt{z} \notin \s{F}_D \, \Longrightarrow \, (\rho \vt{z}) \notin \s{F}_D \quad  \forall \rho >0.
\end{equation}
In other words, if $-\nabla F(\A\x ; \y)$ is not a feasible dual point (i.e., $\exists j$ such that $- \cA_j^\T \nabla F(\A\x ; \y) >0$), then any scaled version of it remains not feasible.

Instead, assuming that $\mathrm{Int}(\s{F}_D) \neq \emptyset$ (i.e., that  the interior of $\s{F}_D$ is  nonempty, see \Cref{rem:interior_FD}), we propose a \textit{dual translation} strategy (by analogy with dual scaling)  defined, for any vector $\vt{t} \in \mathrm{Int}(\s{F}_D)$, as 
\begin{equation}\label{eq:BVLR_dual_up2}
     \Theta(\x) := \Xi_{\vt{t}}(-\nabla F(\A\x;\y))   
\end{equation}
where $\Xi_{\vt{t}}$ is the translation in the direction of vector $\vt{t}$
\begin{equation}
     \Xi_{\vt{t}}(\vt{z}) := \vt{z} + \left( \max_{j \in [n]}  \frac{(\cA_j^\T\vt{z})^+}{|\cA_{j}^\T\vt{t}|} \right)\vt{t}.
\end{equation}
Proposition~\ref{prop:NNLR_dual_trans} shows that this $\Theta$ indeed maps $\x$ onto $\s{F}_D$ and leads to  the desired convergence property. Moreover, similarly to~\eqref{eq:grad}, this dual translation allows one to reuse in the screening test some quantities computed during the primal update.  Note that the additional inner products $\A^\T\vt{t}$ can be pre-computed, which keeps $\A^\T \Theta(\x)$ as cheap as with the standard dual scaling procedure. 

\begin{proposition}[Validity of the dual translation] \label{prop:NNLR_dual_trans}
Let $\Theta$ be defined as in~\eqref{eq:BVLR_dual_up2}. Then, for any primal point  $\x\in \R_{\geq 0}^n$, we have $\Theta(\x) \in \s{F}_D$. Moreover, $\Theta(\x) \to \vtheta^\star$ as $\x \to \x^\star$.
\end{proposition}


\begin{proof}
It is sufficient to show that $\forall \vt{z} \in \R^m$, $\Xi_{\vt{t}}(\vt{z}) \in \s{F}_D$, i.e., $\A^\T \Xi_{\vt{t}}(\vt{z}) \leq 0$.
Denoting  $\epsilon = \max_{k \in [n]}  \frac{(\cA_k^\T\vt{z})^+}{|\cA_{k}^\T\vt{t}|}$, we get,  by definition of $\Xi$ that, $\forall j \in [n]$
\begin{align*}
    \cA_j^\T\Xi_{\vt{t}}(\vt{z}) 
    & = \cA_j^\T(\vt{z} + \epsilon \vt{t})
    = \cA_j^\T\vt{z} + \epsilon \cA_j^\T\vt{t} \\
    & = \cA_j^\T\vt{z} +  \left( \max_{k \in [n]} \frac{(\cA_{k}^\T\vt{z})^+}{|\cA_{k}^\T\vt{t}|} \right) \cA_j^\T\vt{t} \\
    &\leq \cA_j^\T\vt{z} +    \frac{\cA_{j}^\T\vt{z}}{|\cA_{j}^\T\vt{t}|}  \cA_j^\T\vt{t} = 0
\end{align*}
where we used the fact that $\vt{t} \in \mathrm{Int}(\s{F}_D)$, i.e., $\cA_j^\T \vt{t} <0$ $\forall j \in [n]$. Finally, by continuity of $ \nabla F(\A \cdot; \y)$ we get that $$-\nabla F(\A \x; \y) \to  - \nabla F(\A \x^\star; \y) \underset{\eqref{eq:optimality1}}{=} \vtheta^{\star} \text{ as } \x \to \x^\star.$$
Then, the continuity of $\Xi_{\vt{t}}$ together with $\Xi_{\vt{t}}(\vtheta^\star) = \vtheta^\star$ (as $\vtheta^\star \in \s{F}_D$) proves that  $\Theta(\x) \to \vtheta^\star$ as $\x \to \x^\star$.
\end{proof}

\begin{remark}[Comment on  $\mathrm{Int}(\s{F}_D) \neq \emptyset$]\label{rem:interior_FD} One may wonder to what extend such a condition is restrictive in practice. From the expression of $\s{F}_D$ in~\eqref{eq:Dual_set}, we get that $\mathrm{Int}(\s{F}_D) \neq \emptyset$ is equivalent to
\begin{equation}\label{eq:interior_FD}
    \exists \vt{\omega} \in \R^m \; s.t. \; \A^\T \vt{\omega} < 0.
\end{equation}
In other words, the columns of $\A$ must belong to the interior of a half space of $\R^m$ containing the origin on its boundary. Hence, if $\mathrm{Int}(\s{F}_D) = \emptyset$  (i.e.,~\eqref{eq:interior_FD} fails), we have $\vt{0} \in \mathrm{conv}\{\cA_j\}_{j=1}^n$ and the NNLS problem is ill-posed as it admits infinitely many solutions. If furthermore  $\vt{0} \in \mathrm{Int}\left(\mathrm{conv}\{\cA_j\}_{j=1}^n\right)$, we have $\mathrm{cone}\{\cA_j\}_{j=1}^n = \R^m$ meaning that the non-negativity constraint is useless \cite{slawski2013non}.
Note that the latter case corresponds to $\s{F}_D = \{\vt{0}\}$.
To conclude, relevant NNLS problems satisfy  $\mathrm{Int}(\s{F}_D) \neq \emptyset$.
\end{remark}

It remains to discuss how one can determine a direction $\vt{t} \in \mathrm{Int}(\s{F}_D)$. Although there is no systematic approach for a general $\A \in \R^{m \times n}$, this can be achieved on a case-by-case basis for many relevant matrices $\A$.

\begin{proposition}\label{prop:Admissible_A}
The following types of matrix $\A$ ensure that $\mathrm{Int}(\s{F}_D) \neq \emptyset$. Moreover, a vector $\vt{t} \in \s{F}_D$ can be easily computed. \begin{enumerate}
 \setlength{\itemsep}{0pt}
    \item $\A \in \R^{m \times n}$ with $\rank (\A) = n \leq  m$. Then, for any $\vt{b} \in \R_{<0}^n$, all the solutions of $\A^\T \vt{t} = \vt{b}$ (there exists at least one) satisfy $\vt{t} \in \mathrm{Int}(\s{F}_D)$.
    \item $\A \in \R^{m \times m}$ orthogonal (i.e., $\A^\T \A = \Id$). Then, any negative linear combination of the columns of $\A$, i.e., $\vt{t} = \sum_{j \in [n]} \beta_j {\vt{a}}_j$ for $\vt{\beta} \in \R_{< 0}^n$, satisfies $\vt{t} \in \mathrm{Int}(\s{F}_D)$.
    \item $\A \in \R_{\geq 0}^{m \times n}$ with non-negative entries and no column with only zeros. Then, any negative vector $\vt{t} \in \R_{<0}^m$ satisfies $\vt{t} \in \mathrm{Int}(\s{F}_D)$.
    \item $\A \in \R^{m \times n}$ such that $\A^\T \A$ contains a column (say the $j$th) with all entries positive. Then, $\vt{t}=- \cA_j$ satisfies $\vt{t} \in \mathrm{Int}(\s{F}_D)$.
\end{enumerate}
\end{proposition}
\begin{proof}
We prove each case independently.

1. The existence of solutions for $\A^\T \vt{t} = \vt{b}$ comes from the fact that the columns of $\A^\T$ span $\R^n$ (as $\rank (\A) = n$). Then, as  $\vt{b} \in \R_{<0}^n$, we have $\A^\T \vt{t} = \vt{b} < 0$ and thus $\vt{t}  \in \mathrm{Int}(\s{F}_D)$.

2. Due to the orthogonality of $\A$, we have that $\forall i \neq j$, ${\vt{a}}_i^\T {\vt{a}}_j =0$. Hence, for any $\vt{\beta} \in \R_{< 0}^n$, we have
    $$
        \forall i \in [n], \quad  {\vt{a}}_i^\T \left(  \sum_{j \in [n]} \beta_j {\vt{a}}_j \right)  = \beta_i \|{\vt{a}}_i \|_2^2 < 0
    $$
    which shows that $\vt{t} =  \sum_{j \in [n]} \beta_j {\vt{a}}_j  \in \mathrm{Int}(\s{F}_D)$.

3.  Given that $ \forall i \in [n]$, ${\vt{a}}_i \in \R^m_{\geq 0} \backslash \{\mathbf{0} \}$, we have, for any  negative vector $\vt{t} \in \R_{<0}^m$, that
    $
        \forall i \in [n], \;  {\vt{a}}_i^\T \vt{t}   < 0
    $
     which shows that $\vt{t}  \in \mathrm{Int}(\s{F}_D)$.

4.  Direct consequence of the assumption on $\A$.
\end{proof}


\section{Experiments} \label{sec:experiments}

In this paper, we restrict ourselves to the popular bounded-variable ($\vt{l} < \vt{u} < \infty$) and non-negative ($\vt{l} = \vt{0} < \vt{u} = \infty$) least squares problems, i.e., $f(z;y) = \frac12 (z-y)^2$. In this case, the conjugate function is given by $f^*(\theta ; y) = \frac12 ((y+\theta)^2 - y^2)$. 
In the NNLS experiments, we assumed $\A \in \R_{\geq 0}^{m \times n}$, which corresponds to traditional scenarios. As such, unless otherwise stated, the dual translation vector is set to $\vt{t} = - \vt{1}$ (according to~\Cref{prop:Admissible_A}).

We use projected gradient descent~\cite{polyak2015projected} and the Chambolle-Pock primal-dual algorithm \cite{Chambolle2011} to solve the BVLS problem. For NNLS, we consider both the coordinate descent (CD) method of~\cite{franc2005sequential} and the \texttt{lsqnonneg} routine of MATLAB (a variant of the original active set algorithm of~\cite{lawson1974solving}) denoted simply \emph{Active Set} hereafter. The algorithms are stopped when the duality gap falls below $10^{-6}$. For all baselines without screening, the duality gap has been computed offline so as not to impact the measured execution times.

\Cref{sec:Simu} reports results that illustrate how screening performance varies with specific experimental parameters. Then,~\Cref{sec:Real} is devoted to the evaluation of the proposed screening with real datasets.   

\subsection{Understanding screening behaviour}\label{sec:Simu}

\subsubsection{Influence of the saturation ratio.}

\begin{figure}[t]
\begin{center}
\centerline{\includegraphics[width=0.8\columnwidth]{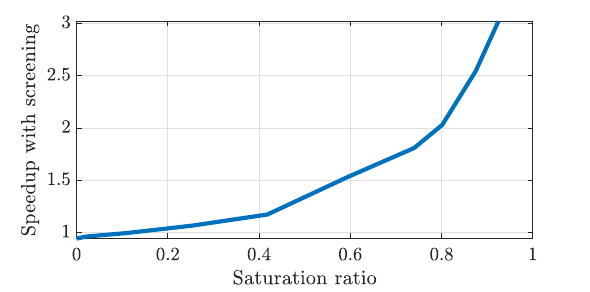}}
\caption{Speedup w.r.t. saturation ratio for a BVLS problem of size $m=4000$, $n=2000$, solved with projected gradient descent. The matrix $\A \in \R^{m \times n}$ and vector $\y \in \R^m$ have been generated according to $a_{ij} \sim \mathcal{N}(0,1)$ and $y_i \sim \mathcal{N}(0,1)$. The saturation ratio is controlled by varying the size of the box $b[-\vt{1},\vt{1}]$.} 
\label{fig:BVLR_Box-size}
\end{center}
\end{figure}

\Cref{fig:BVLR_Box-size} presents the typical evolution of the speedup factor achieved with screening as a function of the saturation ratio (i.e., $s/n$ where $s \in \N$ is the number of saturated components in $\x^\star$). As expected, the higher the number of saturated components, the best the screening performance. Yet, there is a critical value of the saturation ratio under which the computational overhead of screening dominates the acceleration of the primal update, leading to an overall ``speedup'' below 1.

\subsubsection{Influence of the problem parameters.} %

We report in~\Cref{tab:NNLS,tab:BVLS} execution times for NNLS and BVLS problems respectively with increasing size and fixed saturation ratio. For both active set and coordinate descent algorithms, we observe that the screening performance (i.e., the speed improvement) increases with the size  and the level of indeterminacy of the problem. Speedups are obtained consistently and independently of the problem instance (BV or NN) and the chosen solver. Most tested solvers benefit from speedups of around five times, except for the Active Set solver which, by its own nature, is less prone to screening approaches (as they already manipulate reduced sets of coordinates).

\begin{table}[t]
\caption{Execution times and speedup for a NNLS problem with fixed number of rows $m=2000$ and different number of columns~$n$. For each instance of the problem, the matrix $\A \in \R_{\geq 0}^{m \times n}$ has been generated using $a_{ij} = |\eta|$ with $\eta \sim \mathcal{N}(0,1)$. The data vector $\y \in \R^m_{\geq 0}$ has been obtained as $\y = \A \bar{\x} + \vt{\epsilon}$ where $\bar \x \in \R^n_{\geq 0}$ is such that $\|\bar \x\|_0 / n =0.05$ with non-zero entries {$\bar{x}_j$ distributed like $a_{ij}$} and $\vt{\epsilon} \in \R^m$ is such that $\epsilon_i \sim \mathcal{N}(0,1)$.} 
\label{tab:NNLS}
\begin{center} \begin{sc} \begin{tabular}{lcccc}
%
\toprule
                 & $n$        & Baseline  {[}s{]}       & Screening  {[}s{]}    & Speedup    \\ \midrule
\multirow{3}{*}{\rotatebox[origin=c]{90}{\parbox[c]{1.4cm}{\centering Coord. Descent}}} 
     & 1000   & 2.19       &  0.71        & 3.08 \\
     & 2000  &  10.2      &   2.09     &    4.87  \\
     & 4000  &  64.28      &    9.52    &  6.75  \\
     & 6000  &  146.12     &   18.63     &  7.84 \\     
\midrule
\multirow{3}{*}{\rotatebox[origin=c]{90}{\parbox[c]{1.4cm}{\centering Active Set}}} 
     & 1000   & 0.11      &   0.09    &  1.25    \\
     & 2000  &  0.16      &    0.13    &   1.23 \\
     & 4000  &  0.33      &     0.25   &    1.31 \\
     & 6000  &  0.36      &     0.26  &   1.38 \\     
\bottomrule
\end{tabular} \end{sc} \end{center}
\vskip -0.1in
\end{table}

\begin{table}[h!]
\caption{Execution times and speedup for a BVLS problem with $m=1000$ and same setup as in \Cref{tab:NNLS}, except that $\bar{x}_j \sim \mathcal{U}(0,1)$ with bounds $\vt{l}=\vt{0}$ and $\vt{u}=\vt{1}$.}
\label{tab:BVLS}
\begin{center} \begin{sc} \begin{tabular}{lcccc}
%
\toprule
                 & $n$        & Baseline  {[}s{]}       & Screening  {[}s{]}    & Speedup  
                 \\ \midrule
\multirow{3}{*}{\rotatebox[origin=c]{90}{\parbox[c]{1.4cm}{\centering Proj. Grad.}}} 
     & 500  &   9.41  &  1.71  &  5.49  \\ 
     & 1000 &  27.98  &  4.33  &   6.47 \\  
     & 2000 &  127.21 & 18.82  &   6.76 \\  
     & 3000 &  347.05 & 48.46  &   7.16 \\  
\midrule
\multirow{3}{*}{\rotatebox[origin=c]{90}{\parbox[c]{1.4cm}{\centering Primal dual}}} 
     & 500  &  0.26  &  0.08  &   3.41 \\  
     & 1000 &  0.84  &  0.19  &   4.52 \\  
     & 2000 &  2.92  &  0.59  &   4.97 \\  
     & 3000 &  5.20  &  0.95  &   5.48 \\  
\bottomrule
\end{tabular} \end{sc} \end{center}
\vskip -0.1in
\end{table}

\subsubsection{Influence of choice of the dual point.}

In \Cref{fig:NNLS_choice_t}, we report for a NNLS problem the screening ratio (i.e., the number of identified saturated components relatively to the size $n$) as a function of the iteration number, for different choices of the dual translation vector $\vt{t} \in \mathrm{Int}(\mathcal{F}_D)$. Clearly, the choice of $\vt{t}$ affects the screening performance. Although the existence of an optimal choice remains an open question, the reported results allow for some intuition. Indeed, denoting $\vt{a}_{+}$ (resp., $\vt{a}_{-}$) the column of $\A \in \R^{m \times n}_{\geq 0}$ that correlates the most (resp., the least) with all other columns, we can see that setting $\vt{t} = - \vt{a}_{+}$ leads to significantly better screening performance than   $\vt{t} = - \vt{a}_{-}$. Other relevant choices for this example are $\vt{t}= - \vt{1}$ (used in all remaining NNLS examples) and $\vt{t}= - \frac{1}{n} \sum_j \cA_j$. This suggests the conjecture that a relevant $\vt{t}$ should be close to the ``central axis'' of $\mathrm{cone} \{\cA_j\}_{j=1}^n$. 

\begin{figure}
\begin{center}
\centerline{\includegraphics[width=0.8\columnwidth, trim = {0 0 0 0}]{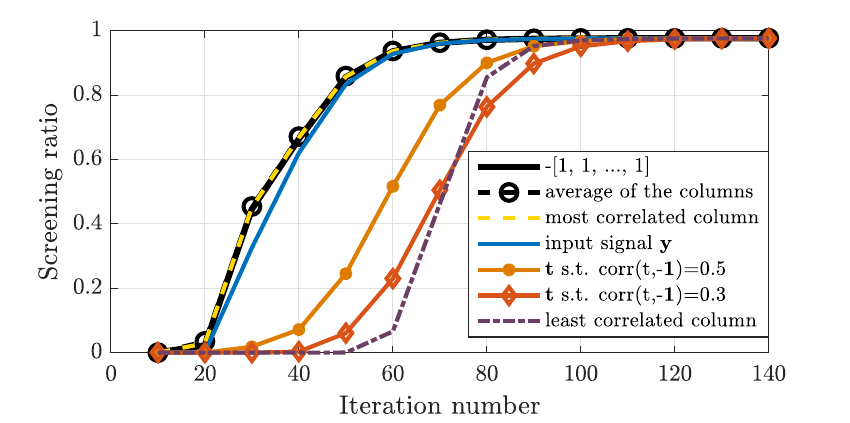}}
\caption{Screening performance with different choices for the translation vector $\vt{t}$ on a NNLS problem with the  NIPS papers dataset (described in \Cref{ssec:NIPSpapers}). } 
\label{fig:NNLS_choice_t}
\end{center}
\vskip -0.2in
\end{figure}

\subsubsection{Limits of screening: oracle dual point.}
To evaluate the practical limits of the proposed screening approach, we perform experiments in which a perfect dual update is performed artificially. The results in \Cref{fig:oracle_dual} show that, although significant acceleration is already obtained with the usual dual update, there is still room for improvement. 
The gap to the optimal bound could be reduced in practice with better dual point estimations. In the NNLS case specifically, this could be achieved by defining better dual directions or maybe even with a completely different approach than the proposed dual translation (see, e.g. \cite{Massias2020}).


\begin{figure}[t]
    \centering
    \begin{minipage}{0.5\textwidth}
        \centering
        \includegraphics[width=\linewidth,trim={0 1.1cm 0.5cm 0} ]{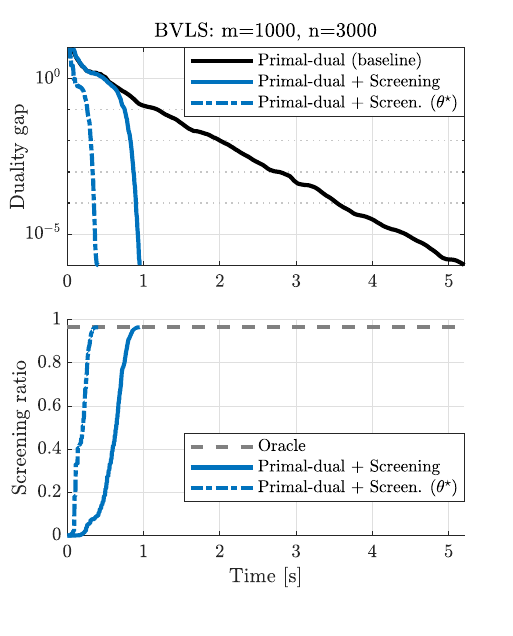}
    \end{minipage}%
    \begin{minipage}{0.5\textwidth}
        \centering
        \includegraphics[width=\linewidth,trim={0 1.1cm 0.5cm 0}]{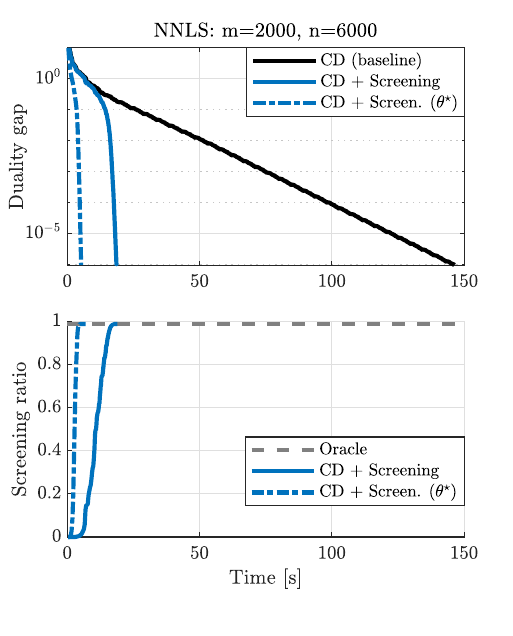}
    \end{minipage}
    \caption{Duality gap convergence (top) and screening ratio (bottom) through time for the simulation setups described in \Cref{tab:NNLS,tab:BVLS}. Left: BVLS problem with primal-dual solver. Right: NNLS problem with coordinate descent solver. The dashed curves correspond to the screening approach artificially informed with an optimal dual point ($\vtheta^\star$), which reaches a speedup of $12.8$ (left) and $27.8$ (right). \label{fig:oracle_dual}}
\end{figure}

\subsection{Performance in applicative scenarios}\label{sec:Real}

\subsubsection{BVLS for hyperspectral unmixing.}

In this experiment, the data vector $\y$ is the observed light reflectance spectrum vector of a random pixel from the Cupitre hyperspectral image~\cite{Jia2007}. The columns of $\A \in \R^{188 \times 342}_{\geq 0}$ are set as the reflectance spectra of pure materials from the USGS High Resolution Spectral Library~\cite{clark2007usgs}, using the same pre-processing as in \cite{Bioucas-Dias2012}. The goal of the regression problem is thus to identify and determine the proportions (so-called abundances) of the materials that compose $\y$. Physical observation constraints dictate that the abundances should lie in the interval $[0, 1]$, leading to a BVLS problem \cite{Bioucas-Dias2012}.

The convergence and screening ratio curves of the projected gradient algorithm with and without screening are presented in \Cref{fig:BVLR_Hyperspectral}. We observe that, as the screening ratio progressively grows, the iterations become faster and the convergence curve eventually detaches from the baseline. 

\subsubsection{NNLS for archetypal analysis.} \label{ssec:NIPSpapers}

The NIPS papers dataset contains word counts from 2484 papers published at the NIPS conference between 1988 to 2003 \cite{NIPSpapers1-17}.
We removed any all-zero columns or rows from the original data matrix and normalized its columns. The input data $\y\in \R_{\geq 0}^{2483}$ is taken as a random sample of the dataset and all remaining samples form the columns of $\A \in \R^{2483\times 14035}_{\geq 0}$. 

The convergence and screening ratio curves of the active set and coordinate descent algorithms are depicted in \Cref{fig:NNLR_NIPSpapers}. While we observe a substantial acceleration for the  coordinate descent, its counterpart for the active set method is more subdued. Yet, in both cases, the proposed screening strategy allows to accelerate the considered solvers.

\section{Conclusion}

In the paper we extended the recently fruitful safe screening framework to the intrinsically distinct family of box-constrained problems. 
Instead of identifying zero coordinates induced by a sparse regularization term, we manage to safely identify saturated coordinates induced by the constraints. 
The main technical challenge in the proposed approach lies in the choice of a dual feasible point which is non-trivial when the box limit is allowed to be unbounded (an example being the widespread NNLS problem). We proposed the simple and efficient \emph{dual translation} procedure to tackle this problem and suggested some practical choices for the translation direction.
Determining an optimal translation direction is actually a challenging problem that deserves further studies---inasmuch as the quality of the dual point can decisively influence the screening performance. 


\begin{figure}[htb!]
\centering
\begin{minipage}{0.5\textwidth}
    \centering
    \includegraphics[width=\linewidth,trim={0 0.6cm 0.5cm 0} ]{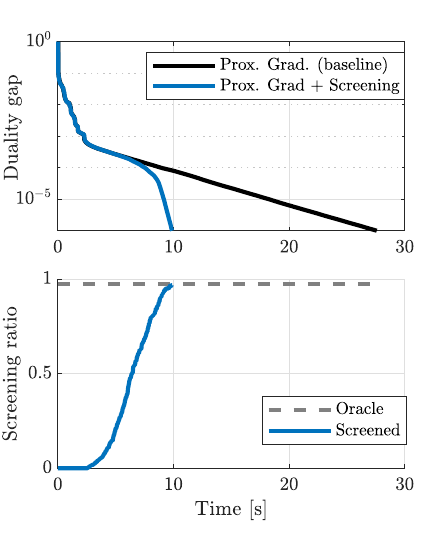}
\end{minipage}%
\begin{minipage}{0.5\textwidth}
    \centering
    \includegraphics[width=\linewidth,trim={0 0.6cm 0.5cm 0}]{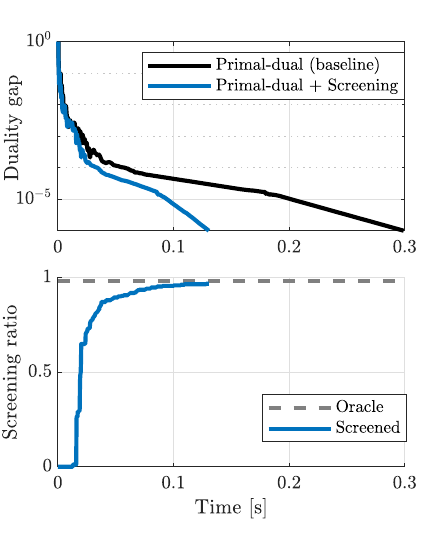}
\end{minipage}
%
%
\caption{Convergence and screening ratio for the BVLS hyperspectral experiment.
Speedups of $2.79$ and  $2.30$ are achieved respectively for the projected gradient (left) and primal-dual (right) solvers. 
\label{fig:BVLR_Hyperspectral}
} 
\begin{minipage}{0.5\textwidth}
    \includegraphics[width=\linewidth,trim={0 0.6cm 0.5cm 0} ]{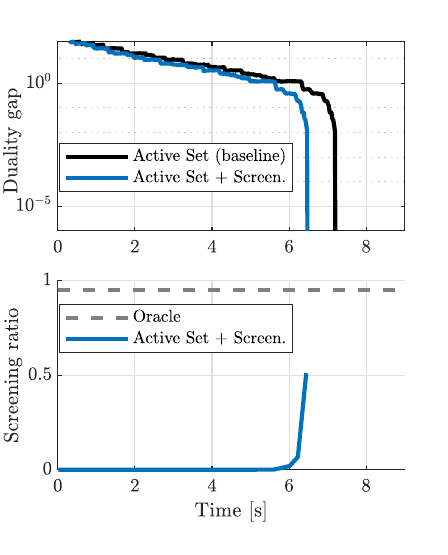}
\end{minipage}%
\begin{minipage}{0.5\textwidth}
    \includegraphics[width=\linewidth,trim={0 0.6cm 0.5cm 0}]{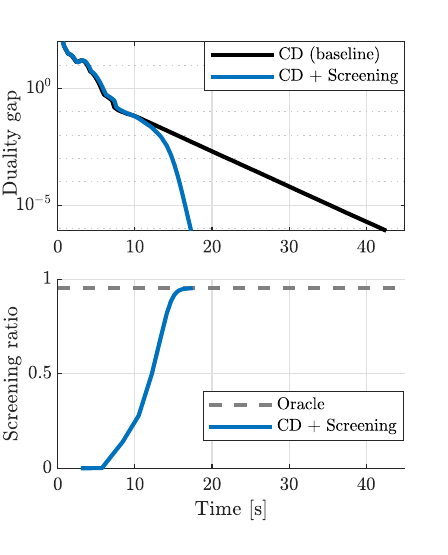}
\end{minipage}
%
%
\caption{ Convergence and screening ratio for the NNLS NIPS papers experiment. Speedups of 2.44 and 1.12 are achieved respectively for the coordinate descent (left) and the active set (right) solvers.} 
\label{fig:NNLR_NIPSpapers}
\end{figure}

\clearpage




\bibliography{PhD}

\begin{thebibliography}{10}
\providecommand{\url}[1]{\texttt{#1}}
\providecommand{\urlprefix}{URL }
\providecommand{\doi}[1]{https://doi.org/#1}

\bibitem{bellavia2006interior}
Bellavia, S., Macconi, M., Morini, B.: An interior point newton-like method for
  non-negative least-squares problems with degenerate solution. Numerical
  Linear Algebra with Applications  \textbf{13}(10),  825--846 (2006)

\bibitem{Bioucas-Dias2012}
Bioucas-Dias, J.M., Plaza, A., Dobigeon, N., Parente, M., Du, Q., Gader, P.,
  Chanussot, J.: Hyperspectral unmixing overview: Geometrical, statistical, and
  sparse regression-based approaches. IEEE Journal of Selected Topics in
  Applied Earth Observations and Remote Sensing  \textbf{5}(2),  354--379
  (2012)

\bibitem{boyer2019representer}
Boyer, C., Chambolle, A., Castro, Y.D., Duval, V., De~Gournay, F., Weiss, P.:
  On representer theorems and convex regularization. SIAM Journal on
  Optimization  \textbf{29}(2),  1260--1281 (2019)

\bibitem{bro1997fast}
Bro, R., De~Jong, S.: A fast non-negativity-constrained least squares
  algorithm. Journal of Chemometrics: A Journal of the Chemometrics Society
  \textbf{11}(5),  393--401 (1997)

\bibitem{Chambolle2011}
Chambolle, A., Pock, T.: A first-order primal-dual algorithm for convex
  problems with applications to imaging. Journal of Mathematical Imaging and
  Vision  \textbf{40}(1),  120–145 (May 2011).
  \doi{10.1007/s10851-010-0251-1}

\bibitem{chen2010nonnegativity}
Chen, D., Plemmons, R.J.: Nonnegativity constraints in numerical analysis. In:
  The birth of numerical analysis, pp. 109--139. World Scientific (2010)

\bibitem{cic10}
Cichocki, A., Amari, S.: Families of {A}lpha- {B}eta- and {G}amma- divergences:
  {F}lexible and robust measures of similarities. Entropy  \textbf{12}(6),
  1532--1568 (June 2010)

\bibitem{clark2007usgs}
Clark, R.N., Swayze, G.A., Wise, R.A., Livo, K.E., Hoefen, T.M., Kokaly, R.F.,
  Sutley, S.J.: {USGS} digital spectral library splib06a. Tech. rep., US
  Geological Survey (2007)

\bibitem{Dantas2021}
Dantas, C., Soubies, E., F{\'e}votte, C.: {Expanding boundaries of Gap Safe
  screening}. Journal of Machine Learning Research (JMLR)  \textbf{22}(236),
  1--57 (2021)

\bibitem{Elvira2020}
Elvira, C., Herzet, C.: Safe squeezing for antisparse coding. IEEE Transactions
  on Signal Processing  \textbf{68},  3252--3265 (2020)

\bibitem{franc2005sequential}
Franc, V., Hlav{\'a}{\v{c}}, V., Navara, M.: Sequential coordinate-wise
  algorithm for the non-negative least squares problem. In: International
  Conference on Computer Analysis of Images and Patterns. pp. 407--414.
  Springer (2005)

\bibitem{NIPSpapers1-17}
Globerson, A., Chechik, G., Pereira, F., Tishby, N.: {Euclidean embedding of
  co-occurrence data}. The Journal of Machine Learning Research  \textbf{8},
  2265--2295 (2007)

\bibitem{Hiriart-Urruty1993a}
Hiriart-Urruty, J.B., Lemar{\'{e}}chal, C.: Convex Analysis and Minimization
  Algorithms {II}. Springer Berlin Heidelberg (1993)

\bibitem{Jia2007}
{Jia}, S., {Qian}, Y.: Spectral and spatial complexity-based hyperspectral
  unmixing. IEEE Transactions on Geoscience and Remote Sensing
  \textbf{45}(12),  3867--3879 (2007). \doi{10.1109/TGRS.2007.898443}

\bibitem{johansson2006application}
Johansson, B., Elfving, T., Kozlov, V., Censor, Y., Forss{\'e}n, P.E.,
  Granlund, G.: The application of an oblique-projected landweber method to a
  model of supervised learning. Mathematical and computer modelling
  \textbf{43}(7-8),  892--909 (2006)

\bibitem{lawson1974solving}
Lawson, C.L., Hanson, R.J.: Solving least squares problems. Prentice-Hall
  Series in Automatic Computation (1974)

\bibitem{Massias2020}
Massias, M., Vaiter, S., Gramfort, A., Salmon, J.: Dual extrapolation for
  sparse glms. Journal of Machine Learning Research  \textbf{21}(234),  1--33
  (2020)

\bibitem{Ndiaye2017}
Ndiaye, E., Fercoq, O., Gramfort, A., Salmon, J.: Gap safe screening rules for
  sparsity enforcing penalties. Journal of Machine Learning Research
  \textbf{18}(128),  1--33 (Nov 2017)

\bibitem{polyak2015projected}
Polyak, R.A.: Projected gradient method for non-negative least square. Contemp
  Math  \textbf{636},  167--179 (2015)

\bibitem{Rockafellar1970}
Rockafellar, R.T.: Convex analysis. Princeton University Press (1970)

\bibitem{slawski2013non}
Slawski, M., Hein, M.: Non-negative least squares for high-dimensional linear
  models: Consistency and sparse recovery without regularization. Electronic
  Journal of Statistics  \textbf{7},  3004--3056 (2013)

\bibitem{stark1995bounded}
Stark, P.B., Parker, R.L.: Bounded-variable least-squares: {A}n algorithm and
  applications. Computational Statistics  \textbf{10},  129--129 (1995)

\bibitem{van2004fast}
Van~Benthem, M.H., Keenan, M.R.: Fast algorithm for the solution of large-scale
  non-negativity-constrained least squares problems. Journal of Chemometrics: A
  Journal of the Chemometrics Society  \textbf{18}(10),  441--450 (2004)

\end{thebibliography}
\bibliographystyle{splncs04} 

\section*{Ethical Statement}
This work contains no sensible or private data. No direct harmful application to society has been identified by the authors.
The research was conducted with a commitment to ethical principles and scientific rigor. 
No human or animal subject were directly involved in this work.
%
The researchers aimed to ensure that the results of this study are accurate, transparent and reproducible, and that they contribute to the advancement of scientific knowledge and to the benefit of society. 
%

\newpage

\appendix

\section{Relation to $\ell_\infty$-regularization} \label{app:L_infinity}

A particular case of \eqref{eq:LR}, obtained by setting $\vt{u} = -\vt{l} = c \vt{1}$, $c>0$, is the $\ell_\infty$-constrained optimization problem
\begin{align}\label{eq:BVLR-inf} 
    \x^\star  \in \argmin_{\x \in \R^n} ~ & \sum_{i=1}^m f([\A \x]_i; y_i) 
    \\ \st ~ & \|\x\|_\infty \leq c. \nonumber
\end{align}

The penalized counterpart of the above problem leads to the $\ell_\infty$-regularized linear regression problem considered by \cite{Elvira2020}%
\begin{align}\label{eq:L_infinity_LR} 
    \x^\star  \in \argmin_{\x \in \R^n} ~ & \sum_{i=1}^m f([\A \x]_i;y_i)  + \lambda \|\x\|_\infty,
\end{align}
where the $\ell_\infty$-norm penalization is controlled by the parameter $\lambda>0$, to which corresponds a certain value of parameter $c$ in the previous constrained formulation. Note that, in \cite{Elvira2020}, only the least-squares case is addressed. 

\section{Dual problems derivation} \label{app:dual_problems}

As a direct application of the Fenchel duality formalism \cite[Theorem 31.3]{Rockafellar1970} 
and the coordinate-wise separability property of the convex conjugate \cite[Ch. X, Prop. 1.3.1]{Hiriart-Urruty1993a} 
we have the following pair of generic primal-dual problems
\begin{align}
    \label{eq:generic_primal}
    \x^\star &\in \argmin_{\x \in \R^n} 
    \overbrace{\sum_{i=1}^m f([\A \x]_i;y_i)}^{F(\A\x;\y)} + \Omega(\x) ,\\
    \label{eq:generic_dual}
    \vtheta^\star &= \argmax_{\vtheta \in \R^m} 
    -\underbrace{\sum_{i=1}^m f^*(-\vthetai;y_i)}_{F^*(-\vtheta;\y)} - \Omega^*(\A^\T \vtheta),
\end{align}
where
$\Omega(\x) = \ind_{\x \in [\vt{l},\vt{u}]}$.
%
To complete the demonstration, we need to calculate the  conjugate function $\Omega^*$ of~$\Omega$.


From the definition of the conjugate function and the separability of $\Omega$, we have
\begin{align*}
     \Omega^*(\vt{w}) & = \sup_{\x \in \R^n} \dotp{\x}{\vt{w}} + \Omega(\x) \\
     & = \sum_{j \in [n]} \sup_{x_j \in [l_j,u_j]} x_j w_j \\
     & =  \sum_{j \in [n]\backslash \s{J}_\infty^u} \sup_{x_j \in [l_j,u_j]} x_j w_j + \sum_{j \in \s{J}_\infty^u} \sup_{x_j > l_j} x_j w_j.
\end{align*}

The first term is a sum of sup of linear functions over compact sets. As such, each sup is attained on the boundary. If $w_j<0$, the sup is attained at $x_j=l_j$ whereas if $w_j>0$, it is attained at $x_j=u_j$. Overall, the value of each sup is given by $l_j (w_j)^- + u_j (w_j)^+$.

The second term is a sum of sup of linear functions over umbounded sets. Each  sup is either attained at $x_j = l_j$ if $w_j <0$ or equals $+\infty$ if $w_j >0$. Each sup is thus equal to $l_j (w_j)^- + \ind_{w_j \leq 0}$.

Combining all these expressions, we obtain
\begin{equation}
    \Omega^*(\vt{w}) = \sum_{j \in [n]} l_j (w_j)^- + \mkern-10mu \sum_{j \in [n]\backslash \s{J}_\infty^u}  u_j (w_j)^+ + \ind_{\vt{w} \leq 0},
\end{equation}
which, with~\eqref{eq:generic_dual}, completes the proof.




\section{First-order optimality conditions}\label{sec:opt_cond}

First-order optimality conditions for the generic primal-dual pair \eqref{eq:generic_primal}-\eqref{eq:generic_dual} are given by \cite[Theorem 31.3]{Rockafellar1970}: 
\begin{align}
    \vtheta^\star & = -\nabla F(\A\x^\star;\y), \\
    \A^\T \vtheta^\star &\in \partial \Omega(\x^\star).
\end{align}
For $F(\A\x;\y) = \sum_{i=1}^m f([\A \x]_i;y_i)$ and $\Omega(\x) = \sum_{j=1}^n \Omega_j(\xj)$ coordinate-wise separable, it simplifies to:
\begin{align}
    &\forall i \in [m],~ \vthetai^\star = -f'([\A\x^\star]_i;y_i),\\
    %
    \label{eq:generic_optimality2}
    &\forall j \in [n],~ \cA_j^\T \vtheta^\star \in \partial \Omega_j(\xj^\star).
\end{align}

Finally, the proof is completed with the expression of the sub-differential of $\Omega_j(\xj) = \ind_{x_j \in [l_j, u_j]}$,
\begin{align*}
    \partial \Omega_j(\xj) = 
    \left\lbrace
    \begin{array}{cl}
         \left(-\infty, 0 \right] & \text{if } ~\xj= l_j \\
         \left[ 0, +\infty \right) & \text{if } ~\xj= u_j \\
         \{0\} & \text{if } ~\xj  \in (l_j, u_j) 
    \end{array} \right.    
\end{align*}
where the second case only occurs when $u_j < \infty$.

\end{document}